\documentclass{article}

\usepackage[preprint]{neurips_2020}

\usepackage[utf8]{inputenc} 
\usepackage[T1]{fontenc}    
\usepackage{hyperref}       
\usepackage{url}            
\usepackage{booktabs}       
\usepackage{amsfonts}       
\usepackage{amsmath}
\usepackage{amsthm}
\usepackage{nicefrac}       
\usepackage{microtype}      
\usepackage{algpseudocode}
\usepackage{algorithm}

\newtheorem{theorem}{Theorem}

\newtheorem{definition}{Definition}
\newtheorem{corollary}{Corollary}

\newtheorem{lemma}{Lemma}
\usepackage{appendix}
\usepackage{bbm}
\usepackage{mathtools}
\usepackage{array}
\usepackage{tabularx}
\usepackage{xcolor}
\usepackage{ifthen}
\usepackage{graphicx}
\usepackage{caption}
\usepackage{subcaption}
\DeclarePairedDelimiter\ceil{\lceil}{\rceil}
\DeclarePairedDelimiter\floor{\lfloor}{\rfloor}

\newcommand{\R}{\mathbb{R}}
\newcommand{\N}{\mathbb{N}}

\newcommand{\ra}{\rightarrow}

\newcommand{\Exp}[1]{\mathbb{E}\left[#1\right]} 

\newcommand{\ashutosh}[1]{\textcolor{orange}{#1}}
\newcommand{\alternate}[1]{\textcolor{purple}{#1}}

\newcommand{\B}{\mathcal{B}}
\newcommand{\SG}{\mathcal{SG}}
\newcommand{\SE}{\mathcal{SE}}
\newcommand{\G}{\mathcal{G}}
\newcommand{\M}{\mathcal{M}}

\newcommand{\U}{\mathcal{U}}
\newcommand{\W}{\mathcal{W}}

\newboolean{showcomments}
\setboolean{showcomments}{true}
\newcommand{\jk}[1]{\ifthenelse{\boolean{showcomments}} {\textcolor{red}{(JN says: #1)}} {} }
\newcommand{\ka}[1]{\ifthenelse{\boolean{showcomments}} {\textcolor{red}{(KA says: #1)}} {} }
\newcommand{\kj}[1]{\ifthenelse{\boolean{showcomments}} {\textcolor{red}{(KJ says: #1)}} {} }
\newcommand{\ak}[1]{\ifthenelse{\boolean{showcomments}} {\textcolor{red}{(AK says: #1)}} {} }

\newcommand{\ignore}[1]{}

\title{Bandit algorithms: Letting go of logarithmic regret for
  statistical robustness}

\author{%
  Kumar Ashutosh \\
  Department of Electrical Engineering\\
  IIT Bombay, Mumbai \\
  \texttt{kumar.ashutosh@iitb.ac.in}
   \And
   Jayakrishnan Nair \\
   Department of Electrical Engineering \\
   IIT Bombay, Mumbai \\
   \texttt{jayakrishnan.nair@ee.iitb.ac.in}
   \AND
   Anmol Kagrecha \\
   Department of Electrical Engineering \\
   IIT Bombay, Mumbai  \\
   \texttt{akagrecha@gmail.com} 
   \And
   Krishna Jagannathan \\
   Department of Electrical Engineering \\
   IIT Madras, Chennai \\
   \texttt{krishnaj@ee.iitm.ac.in} 
}

\begin{document}

\maketitle

\begin{abstract}

  We study regret minimization in a stochastic multi-armed bandit
  setting, and establish a fundamental trade-off between the regret
  suffered under an algorithm, and its statistical robustness. Considering broad classes of underlying arms' distributions, we show
  that bandit learning algorithms with logarithmic regret are \emph{always
  inconsistent}, and that \emph{consistent} learning algorithms always suffer a super-logarithmic regret. This result highlights the inevitable statistical fragility of all `logarithmic
  regret' bandit algorithms available in the literature---for instance,
  if a UCB algorithm designed for $\sigma$-subGaussian distributions
  is used in a subGaussian setting with a mismatched variance
  parameter, the learning performance could be inconsistent. Next, we show a positive result: statistically robust and consistent
  learning performance is attainable if we allow the regret to be
  \emph{slightly worse} than logarithmic. Specifically, we propose three  classes of distribution oblivious algorithms that achieve an
  asymptotic regret that is arbitrarily close to logarithmic.
\end{abstract}

\section{Introduction}
\label{headings}


The stochastic multi-armed bandit (MAB) problem seeks to identify the best among an available basket of options (a.k.a., arms), each
characterized by an unknown probability distribution. Classically,
these probability distribution represent rewards, and the best arm is
defined as the one associated with the largest average reward. The
learning algorithm, which chooses (a.k.a., pulls) one arm per decision
epoch, identifies the best arm via experimentation---each pull of an
arm yields one sample from the underlying reward distribution. One
classical performance metric is \emph{regret}, which evaluates an
algorithm based on how often it pulls sub-optimal arms.

The standard approach towards algorithm design for regret minimization is as follows. First, it is assumed that the arm reward distributions belong to a specific parametric class---for example, the class of bounded distributions with support contained in $[0,b],$ or the class of $\sigma$-subGaussians. Next, algorithms are proposed for such specific parametric distribution classes, often making explicit use of the parameters (such as $b$ or $\sigma$) corresponding to the parametric distribution class. Finally, logarithmic regret guarantees are proved for such algorithms, by utilising exponential concentration inequalities (such as Hoeffding's inequality or sub-Gaussian concentration) for that parametric distribution class.

For distribution classes such as $\sigma$-subGaussians, a logarithmic regret guarantee may not be so surprising, because such distributions enjoy exponential concentration bounds. On the other hand, when  dealing with heavy-tailed arms' distributions, it is not clear that a logarithmic regret is achievable. This is because heavy-tailed distributions (such as Pareto) are characterised by a high degree of variability, and their empirical mean estimators do not enjoy exponential concentration in the sample size. Somewhat surprisingly, a logarithmic regret guarantee was shown to be attainable in \cite{bubeck} using a truncated mean estimator, for distributions satisfying a bounded moment condition. While this approach \cite{bubeck} can handle heavy-tailed as well as light-tailed distributions, the algorithm still needs to know the moment bounds.

As such, a logarithmic regret guarantee has been shown to hold in a broad range of stochastic bandit settings. At this point, it is perhaps not an exaggeration to suggest that a logarithmic regret is regarded as a `default performance expectation' from `good' stochastic bandit learning algorithms. The present paper challenges this perceived sanctity of logarithmic regret, in the context of low-regret learning of stochastic MABs. We show that bandit algorithms that enjoy a logarithmic regret guarantee cannot be statistically robust.



\textbf{Our contributions:} We make two key contributions in this paper.

First, we show that bandit algorithms that enjoy a logarithmic regret guarantee are fundamentally fragile from a statistical standpoint. Equivalently, we show that statistically robust algorithms necessarily incur super-logarithmic regret.  Here, an algorithm is said to be statistically robust if it exhibits \emph{consistency}, i.e., the regret scales slower than any power-law, over a suitably broad class of MAB instances.  

For example, consider an algorithm with logarithmic regret designed for $\sigma$-subGaussian arms. When this algorithm is used in a `mismatched' bandit instance, say with $\sigma'$-subGaussian arms ($\sigma' > \sigma$),  the learning performance can be \emph{inconsistent.} That is, the regret suffered by the algorithm in the mismatched instance could have a power-law scaling in the time horizon. This is of practical concern, since the parameters that define the space of arms' distributions (usually in the form of support/moment bounds) are often themselves estimated from limited data samples, and are therefore prone to errors.


Our second contribution is a positive result: we show that statistically robust learning is achievable if we are willing to tolerate a `slightly-worse-than-logarithmic' regret in the time horizon. Specifically, we propose three classes of algorithms that (i) are \emph{distribution oblivious} (i.e., they require no prior information about the arm distribution parameters), and (ii) incur a regret that is slightly
super-logarithmic. The first algorithm class offers this guarantee
over subexponential (a.k.a., light-tailed) instances. The
latter two are designed to work robustly for general distribution instances (excepting some pathological ones). 

In all three algorithms, the asymptotic regret guarantee is controlled by a certain slow-growing scaling function that is used to to define confidence bounds. A more slowly growing scaling function makes the regret asymptotically closer
to logarithmic, but at the expense of a potential degradation in
performance for shorter horizons. Furthermore, the regret for shorter
horizon-lengths can be improved by incorporating (noisy) prior
information about the reward distributions into the scaling function,
without compromising on statistical robustness.

\textbf{Related literature:} There is a vast literature on the
regret minimization for the stochastic MAB problem; we refer the
reader to the textbook treatments
\cite{Bubeck2012,lattimore2018bandit}. However, to the best of our
knowledge, the issue of statistical robustness and its connection to
logarithmic regret has not been explored before. 

We are aware of only two other works that address statistical
robustness in context of bandit algorithms, both of which consider the
fixed budget pure exploration setting. For the best arm identification
problem, stastistically robust algorithms have been demonstrated
recently in~\cite{kagrecha2019}. For thresholding bandit problem, the
algorithm proposed in \cite{Locatelli2016} is
\emph{distribution-free}, i.e., the algorithm does not require
knowledge of the $\sigma$ parameter defining the space of
$\sigma$-subGaussian rewards.

The remainder of this paper is organized as follows. We introduce some
preliminaries and define the MAB formulation in
Section~\ref{headings}. The trade-off between statistical robustness
and logarithmic regret is established in
Section~\ref{lower-bounds-section}. Our statistically robust
algorithms and their performance guarantees are presented in
Section~\ref{upper-bounds}, and we report the results of some
numerical experiments in Section~\ref{expt-analysis}. An appendix,
containing proofs of stated results, as well as details omitted from
the main body of the paper due to space constraints, is uploaded
separately as the `supplementary material' document.

\section{Model and Preliminaries}
\label{headings}

In this section, we introduce some preliminaries and formally define
the MAB formulation.

\subsection{Preliminaries}

We begin by introducing the classes of reward distributions we will
work with in this paper.
\begin{itemize}
\item $\B([a,b])$ denotes the set of bounded distributions with
  support contained in $[a, b]$. The set of all bounded distributions
  is denoted by $\B.$
\item We use $\SG(\sigma),$ for $\sigma > 0,$ to denote
  $\sigma$-subGaussian distributions, and $\SG$ to denote all
  subGaussian distributions.
\item We denote $\SE(v, \alpha),$ for $v, \alpha > 0,$ to denote the
  following class of subexponential distributions:
  $$ \SE(v, \alpha) = \left\{ F : \int e^{\lambda(x-\mu(F))} dF(x) \leq e^{ \frac{v^2 \lambda^2}{2}}
    \text{ for all } |\lambda| < \frac{1}{\alpha}\right\},$$ where
  $\mu(F)$ denotes the mean of~$F.$ The class of all subexponential
  distributions is denoted by $\SE.$ Distributions in $\SE$ are also
  commonly referred to as \emph{light-tailed}, and those not in $\SE$
  are called \emph{heavy-tailed} (see \cite{Foss2011}).
\item For $\epsilon,B > 0,$ let $\G(\epsilon,B)$ denote the set of
  distributions whose $(1+\epsilon)^{th}$ absolute moment is upper
  bounded by $B$, i.e.,
  $$\G(\epsilon,B) = \left\{F:\ \int |x|^{1+\epsilon} dF(x) \leq B \right\}.$$
  In the MAB literature, $\G(\epsilon,B)$ is often used as the class
  of reward distributions in order to allow for heavy-tailed rewards
  (see, for example, \cite{bubeck,Yu2018PureEO}). Finally, the union
  of the sets $\G(\epsilon,B)$ over $\epsilon,B > 0$ is denoted by
  $\G:$
  $$\G = \left\{F:\ \int |x|^{1+\epsilon} dF(x) < \infty \text{ for some } \epsilon > 0\right\}.$$
  $\G$ is the most general space of reward distributions one can work
  with in the context of the MAB problem---it contains all
  light-tailed distributions and most heavy-tailed distributions of
  interest.
\end{itemize}
Note that $\B \subset \SG \subset \SE \subset \G.$  We also recall the Kullback-Leibler divergence (or relative
entropy) between distributions $F$ and $F'$: \begin{align*}
    D(F, F') = \int \log \left( \frac{dF(x)}{dF'(x)} \right) dF(x),
\end{align*}
where $F$ is
absolutely continuous with respect to $F'.$ 

Much of the vast literature on MAB problems assumes that the
reward distributions lie in specific parametric subsets of $\B,$ $\SG,$ $\SE,$ or
$\G;$ for example $\B([0,1]),$ $\SG(1),$ $\G(1,B)$ etc. Further, the
parameter(s) corresponding to these subsets are `baked' into the
algorithms. While this approach guarantees strong performance over the parametric distribution subset under consideration
(logarithmic regret, in the classical regret minimization framework), it is highly fragile to uncertainty in these parameters. Indeed, as we demonstrate in Section~\ref{lower-bounds-section}, any algorithm that enjoys logarithmic regret for a parametric subset of a distribution class \emph{must be inconsistent} over the entire distribution class---specifically, when there is a parameter mismatch, the regret suffered could have a power-law scaling in the time horizon. In Section~\ref{upper-bounds}, we propose bandit algorithms that are statistically robust, but incur (slightly) superlogarithmic regret.


\subsection{Problem formulation}

Consider a multi-armed bandit (MAB) problem with $k$ arms. Let $\M$ be a distribution class (such as $\B,\SG$ etc.) An instance
$\nu = (\nu_i,\ 1 \leq i \leq k)$ of the MAB problem is defined as an element of $\M^k,$ where $\nu_i\in\M$ is the distribution corresponding to arm $i$. Let $\mu_i$ denote the mean reward
associated with arm~$i,$ i.e., $\mu_i$ is the expected value of a random variable distributed according to $\nu_i.$ An \emph{optimal} arm is an arm that maximizes
the mean reward, i.e., one whose mean reward equals
$\mu^{*} = \max_{1 \leq i \leq k} \mu_i.$ The \emph{sub-optimality
  gap} associated with arm~$i$ is defined as
$\Delta_i := \mu^* - \mu_i.$ 

In this paper, our goal is to minimize \emph{regret}. Formally, under
the a policy (a.k.a., algorithm)~$\pi,$ let~$T_i(n)$ denote the number
of times $i^{th}$ arm has been pulled after $n$ rounds. The regret
$R_n(\pi,\nu)$ associated with
the policy $\pi$ after~$n$ rounds is defined as
$$R_n(\pi,\nu)= \sum_{i = 1}^n \Delta_i \Exp{T_i(n)}.$$

An algorithm is said to be \textit{consistent}
over~$\M^k$ if, for all instances $\nu \in \M^k,$ the regret satisfies $R_n(\pi,\nu) = o(n^a)$ for
all $a > 0$ (see \citet{lattimore2018bandit}). For example, an algorithm that guarantees polylogarithmic regret over \emph{all} instances in $\M^k$ is consistent over~$\M^k.$ On the other hand, if an algorithm suffers $O(n^a)$ regret for some $a>0$ and \emph{some} instance in $\M^k,$ then the algorithm is inconsistent over~$\M^k.$


\section{Impossibility of logarithmic regret for statistically robust
  algorithms}
\label{lower-bounds-section}

In this section, we shed light on a fundamental conflict between
logarithmic regret and statistical robustness. Recall that in
classical MAB formulations, it is assumed that arm reward
distributions lie in, say $\B([0,b])$ or $\SG(\sigma).$ In such cases,
algorithms that exploit this parametric information (i.e., the value
of $b$ in the former case and the value of $\sigma$ in the latter) are
known that achieve $O(\log(n))$ regret, where $n$ denotes the
horizon. The celebrated UCB family of algorithms is a classic example
\citep{lattimore2018bandit}. In this section, we ask the question: Are
these algorithms robust with respect to the parametric information
`baked' into them?  Our main result of this section answers this
question in the negative. Specifically, we show that statistically
robust algorithms (i.e., algorithms that maintain consistency over an
entire class of distributions) necessarily incur super-logarithmic
regret. In other words, algorithms that enjoy a logarithmic regret
guarantee over a particular parametric sub-class of reward
distributions are \emph{not} statistically robust.

\ignore{The analysis of lower bounds is important to understand the
  complexity of the problem at hand. Lower bounds provide a
  theoretical limit to the performance of any algorithm. In our MAB
  setting, we are interested in instance specific lower bound. By
  instance specific bound we mean the performance of any algorithm for
  a given instance $\nu \in \M^k.$ Here, we use the cumulative regret
  as the measure of performance of an algorithm. Our main interest is
  to establish the impossibility of logarithmic regret algorithm for
  oblivious classes of algorithms over $\B,$ $\SG,$ or $\G.$}
\begin{theorem}
  \label{lower-bound-main-statement}
  Let $\M \in \{\B,\SG,\SE,\G\}.$ For any algorithm $\pi$ that is
  consistent over $\M^k,$ and any instance $\nu \in \M^k,$
   \begin{align*}
     \lim_{n \to \infty} \frac{R_n(\pi, \nu)}{\log(n)} = \infty.
   \end{align*}
\end{theorem}

The proof of Theorem~\ref{lower-bound-main-statement} is provided in
Appendix~\ref{lower-bounds-appendix}. The crux of the argument is as
follows. Given an MAB instance $\nu \in \M^k,$ the expected number of
pulls $\Exp{T_i(n)}$ of any suboptimal arm~$i$ over a horizon of~$n$
pulls, under any algorithm that is consistent over $\M^k,$ is lower
bounded as
$$\liminf_{n \ra \infty} \frac{\Exp{T_i(n)}}{\log(n)} \geq \frac{1}{d_i},$$
where
$d_i = \inf_{\nu'_i \in \M}\{D(\nu_i,\nu'_i):\ \mu(\nu'_i) >
\mu^*(\nu)\}$ (see \citet[chap.~16]{lattimore2018bandit}). Informally,
$d_i$ is the smallest perturbation of $\nu_i$ in relative entropy
sense that would make arm~$i$ optimal. The proof of
Theorem~\ref{lower-bound-main-statement} follows by showing that when
$\M$ is $\B,$ $\SG,$ $\SE$ or $\G,$ we have $d_i = 0$ for all
suboptimal arm of any instance. In other words, given any distribution
$\eta \in \M,$ there exists another distribution $\eta' \in \M$ such
that $\mu(\eta')$ is arbitrarily large, even while $D(\eta,\eta')$ is
arbitrarily small.

Theorem~\ref{lower-bound-main-statement} highlights that classical
bandit algorithms are not robust with respect to uncertainty in
support/moment bounds. For example, consider any algorithm $\pi$ that
guarantees logarithmic regret over $\SG(1)$ (for example, the
algorithms presented in Chapters~7--9 in \citet{lattimore2018bandit}).
Theorem~\ref{lower-bound-main-statement} implies that all such
algorithms are \emph{inconsistent} over $\SG.$ This reveals an
inherent fragility of such algorithms---while they might guarantee
good performance over the specific parametric sub-class of reward distributions they are designed for, they are not robust to uncertainty with respect to the parameters that specify the distribution class.

Having shown that robust algorithms cannot achieve logarithmic regret,
in the following section, we present statistically robust algorithms
for $\SE,$ and $\G.$ (Of course, an algorithm that is robust over
$\SE$ is also robust over $\B$ and $\SG$). Specifically, these
algorithms attain a regret that is \emph{slightly} superlogarithmic,
while remaining consistent over $\SE$ and $\G$ respectively.


\section{Statistically robust algorithms}
\label{upper-bounds}

In this section, we demonstrate how statistical robustness can be
achieved by allowing for \emph{slightly superlogarithmic} regret. In
particular, we propose algorithms that are \emph{distribution
  oblivious}, i.e., they do not require any prior information about
the arm distributions in the form of support/moment/tail bounds. By
suitably choosing a certain scaling function that paramterizes the
algorithms, the associated regret can be made arbitrarily close to
logarithmic (in the time horizon). However, this is not an entirely
`free lunch'---tuning the scaling function for stronger asymptotic
regret guarantees can affect the regret for moderate horizon
values. Interestingly though, this trade-off between asymptotic and
short-horizon performance can be tempered by incorporating (noisy)
prior information about support/moment bounds on the arm distributions
into the scaling functions, while maintaining statistical robustness.


We propose three distribution oblivious algorithms for robust regret
minimization in this section. The first, which we call Robust Upper
Confidence Bound (R-UCB) algorithm is suitable for subexponential
(light-tailed) instances. (An instance is said to be light-tailed if
all arm distributions are light-tailed). It uses the empirical average
as an estimator for the mean reward, and uses a confidence bound that
that is a suitably (and robustly) scaled version of the typical
non-oblivious confidence bounds in UCB algorithms.

Next, to deal with the most general class $\G$ of reward
distributions, we propose another algorithm, called R-UCB-G, which
uses \emph{truncated} mean estimators. Empirical averages, which
provide good estimates of the mean for light-tailed arms, can deviate
significantly from the true mean for heavy-tailed arms. To control the
`high variability' in the sample values, a truncated mean estimator is
typically used; see for example, \cite{bubeck,Yu2018PureEO}.  The
truncation parameter in R-UCB-G is scaled with time suitably to
provide statistical robustness. Desirably, both R-UCB \& R-UCB-G are
\emph{anytime} algorithms, and have provable regret guarantees.

Another technique for mean estimation that works well under excessive
variability in the sample values is the Median of Means approach
\citep{bubeck}. We design a statistically robust anytime algorithm
over~$\G^k$ using this approach; due to space constraints, the
algorithm and its performance characterization are presented in
Appendix~\ref{mom-section}.

Before we describe the algorithms, we define the following class of
functions which serve as scaling functions for both algorithms.

\begin{definition}
\label{ref:f}
  A function~$f:\mathbb{N} \ra (0,\infty)$ is said to be \emph{slow
    growing} if
  $$f(t + 1) \geq f(t)\ \forall\ t \in \N, \quad
  \lim_{t \ra \infty} f(t) = \infty, \quad \lim_{t \ra \infty}
  \frac{f(t)}{t^a} = 0 \ \forall \ a > 0.$$
\end{definition}

\subsection{Robust Upper Confidence Bound algorithm for
  light-tailed instances}

\begin{algorithm}[t]
  \caption{R-UCB}
  \label{robust-light-UCB-algo}
  \textbf{Input} $k$ arms, slow growing scaling function $f$
  \begin{algorithmic}
    \For  {$t=1$ to $k$}
    \State Pull arm with index $i = t-1$ and observe reward $R_t$
    \State Update $\hat{\mu}(i, u_i) \leftarrow r$, $u_i \leftarrow 1$
    \EndFor 
    \For { $t=k+1, k+2, \dots$ } 
    \State Calculate the upper confidence bound as $$\U(i, u_i, t) = \hat{\mu}(i, u_i) + \underbrace{\sqrt{\frac{f(t)\log(t)}{u_i}}}_{\W(u_i, t)} $$
	\State Pull arm $i$ maximizing $\U(i, u_i, t)$ and observe reward $R_t$
	\State Update empirical average $\hat{\mu}(i, u_i)$ and $u_i\leftarrow
        u_i+1$
    \EndFor  
  \end{algorithmic}
\end{algorithm}

The R-UCB algorithm is presented in
Algorithm~\ref{robust-light-UCB-algo}. The only structural difference
between R-UCB and the classical UCB algorithm is in the definition of
the upper confidence bound---under R-UCB, the confidence
width~$\W(u_i, t)$ for arm~$i$ at time~$t,$ where $u_i$ denotes the
number of pulls of arm~$i$ prior to time~$t,$ is scaled by a slow
growing function~$f.$ This simple scaling provides statistical
robustness over light-tailed instances, as established in
Theorem~\ref{thm:sr-ucb-1} below. We prove the consistency of R-UCB
over all subexponential instances, albeit with superlogarithmic
regret. We also provide stronger guarantees for subgaussian instances.

\begin{theorem}
  \label{thm:sr-ucb-1}
  Consider the algorithm R-UCB with a specified slow growing scaling
  function~$f.$ For an instance $\nu \in \SE(v, \alpha)^k,$ there
  exists threshold $t^{\SE}_{min}(v, \alpha)$ such that for
  $t > t^{\SE}_{min}(v, \alpha),$ the regret under R-UCB satisfies
  \begin{equation}
    \label{eq:rucb-se}
    R_t(\nu) \leq \sum_{i:\Delta_i>0} \left( f(t)\log(t) ~ \max \left\{
        \frac{4}{\Delta_i}, \Delta_i \left( \frac{\alpha}{v^2} \right)^2
      \right\} + 4\Delta_i \right).
  \end{equation}
  For an instance $\nu \in \SG(\sigma)^k,$ there exists a threshold
  $t^{\SG}_{min}(\sigma)$ such that for $t > t^{\SG}_{min}(\sigma),$
  the regret under R-UCB satisfies
  \begin{equation}
    \label{eq:rucb-sg}
    R_t(\nu) \leq \sum_{i:\Delta_i>0} \left(
      \frac{4f(t)\log(t)}{\Delta_i} + 4\Delta_i \right).
  \end{equation}
\end{theorem}

  

The key take-aways from Theorem~\ref{thm:sr-ucb-1} are as follows.
\begin{itemize}
\item R-UCB is clearly consistent over~$\SE^k,$ but the regret
  guarantee is super-logarithmic, as demanded by
  Theorem~\ref{lower-bound-main-statement}.
\item R-UCB is distribution oblivious in the sense that it does not
  need the parameters $v,\alpha$ in the implementation. However, the
  stated regret guarantee holds for $t$ greater than an
  instance-dependent threshold $t^{\SE}_{min}(v, \alpha)$---this is
  because the confidence width needs to be large enough for certain
  concentration properties to hold. Explicit characterization of the
  threshold $t^{\SE}_{min}(v, \alpha)$, along with (weaker) regret
  bounds for $t$ less than this threshold, are provided in
  Appendix~\ref{proof-sr-ucb-1}. 
\item Choosing $f$ to be `slower' growing leads to better asymptotic
  regret guarantees, but increases the threshold $t_{min}.$ This
  implies a trade-off between asymptotic and short-horizon performance
  in a purely oblivious setting. However, (noisy) prior information
  about the class of arm distributions can be incorporated into the
  choice of scaling function~$f$ to dilute this tradeoff. For example,
  if it is believed that the arm distributions are
  $\sigma$-subgaussian, then one may set $f(t) = 8\sigma^2 + h(t),$
  where $h(\cdot)$ is slow growing; this choice of motivated by the
  observation that for the well known (non-robust) $\alpha$-UCB
  algorithm \citep{Bubeck2012}, $f$ would be replaced by
  $2\alpha \sigma^2,$ $\alpha > 1$ for $\sigma$-subGaussian arms. This
  choice would make $t^{\SG}_{min}$ small if the arms are
  $\sigma'$-subgaussian, where $\sigma' \approx \sigma,$ while still
  providing statistical robustness to the reliability of this prior
  information; see Appendix~\ref{proof-sr-ucb-1}. We also illustrate
  this phenomenon in our numerical experiments in
  Section~\ref{expt-analysis}.
\item Stronger performance guarantees are possible for the
  subclass~$\SG^k.$ Indeed, given that
  $\SG(\sigma) \subset \SE(\sigma,\alpha)$ for all $\alpha > 0,$ the
  guarantee~\eqref{eq:rucb-sg} is stronger than~\eqref{eq:rucb-se} for
  $\nu \in \SG(\sigma)^k.$
\end{itemize}
The proof of Theorem~\ref{thm:sr-ucb-1} is provided in
Appendix~\ref{proof-sr-ucb-1}.

\subsection{Robust Upper Confidence Bound algorithm for
  arbitrary instances}
 \label{ssec:r-ucb-2}

The R-UCB algorithm discussed above is robust to parametric
uncertainties, and guarantees `slightly-worse-than-logarithmic' regret
for any light-tailed bandit instance. However, one could argue that
R-UCB is still not \emph{truly} robust---after all, how can we be
certain in a practical scenario that there are no heavy-tailed arms
involved? From a viewpoint of applications such as financial
portfolios and insurance, heavy-tailed distributions are ubiquitously
used in modelling. Therefore there is a compelling case for handling
heavy-tailed as well as light-tailed arms' distributions within a
common, statistically robust framework.

In this section, we propose a truly robust algorithm for the most
general setting, i.e., for bandit instances in~$\G^k.$ We recall that
the class $\G$ demands only the boundedness of the
$(1+\epsilon)$-moment for some $\epsilon>0.$ This is only mildly more
demanding than the finiteness of the mean,\footnote{Distributions with
  finite mean that do not belong to $\G$ are quite pathological, and
  are of little practical interest.} which is necessary for the MAB
problem to be well-posed.

Once the restriction to light-tailed reward distributions is removed,
more sophisticated estimators than empirical averages are required;
this is because empirical averages are highly sensitive to (relatively
frequent) outliers in heavy-tailed data. One such approach is to use
truncation-based estimators (see, for example, \cite{bubeck}), which
offer lower variability at the expense of a (controllable) bias. The
R-UCB-G algorithm, stated formally as
Algorithm~\ref{robust-heavy-UCB-algo}, uses a truncation-based
estimator in conjunction with a robust scaling of the confidence
bound. Note that the same scaling function~$f$ is used for both
truncation as well for scaling the confidence bound.

R-UCB-G provides the following performance guarantee over instances in~$\G^k.$ To the best of our knowledge, this is the first time a single algorithm has been shown to provide provable regret guarantees in such generality. 
\begin{algorithm}[t]
  \caption{R-UCB-G}
  \label{robust-heavy-UCB-algo}
  \textbf{Input} $k$ arms, slow growing scaling function $f$ taking
  values in $(1,\infty)$
  \begin{algorithmic}
     \State \textbf{Initialize} $\mathcal{R}_i = \{~\}$, $u_i=0$ for all arm $i$
    \For  {$t=1$ to $k$}
    \State pull arm with index $i = t-1$ and observe reward $r$
    \State Append r to $\mathcal{R}_i$ and update $u_i \leftarrow u_i + 1$
    \EndFor 
    \For { $t=k+1, k+2, \dots$ } 
    \State Calculate the upper confidence bound as $$\U(i, u_i, t) = \underbrace{\frac{1}{u_i} \sum_{X \in \mathcal{R}_i} X\mathbbm{1}_{ \left \{ |X| \leq f(t) \right \} }}_{\hat{\mu}(i, u_i, t)} + \underbrace{\frac{1}{\log(f(t))} + \frac{16f(t)\log(t)}{u_i}}_{\W(u_i, t)} $$
	\State Pull arm $i$ maximizing $\U(i, u_i, t)$ and observe reward $R_t$
	\State Append $R_t$ to $\mathcal{R}_i$ and update $u_i \leftarrow u_i+1$
    \EndFor  
  \end{algorithmic}
\end{algorithm}

\begin{theorem}
  \label{thm:sr-ucb-2}
  Consider the algorithm R-UCB-G with a specified slow growing scaling
  function~$f$ taking values in $(1,\infty).$ For an instance
  $\nu \in \mathcal{G}(\epsilon,B)^k,$ there exists a threshold
  $t_{min}(\epsilon,B)$ such that for $t > t_{min}(\epsilon,B),$ the
  regret under R-UCB-G satisfies
  $$R_t(\nu) \leq \sum_{i:\Delta_i>0} \left( \frac{32f(t)\log(t)}{1 -
      \frac{2}{\Delta_i \log(f(t))}} + 4\Delta_i \right).$$
\end{theorem}

The performance guarantee of R-UCB-G is structurally similar to that
for R-UCB: The algorithm is consistent, with a super-logarithmic
regret that is dictated by the growth of the scaling function~$f.$
Moreover, while slowing the growth of~$f$ improves the asymptotic
regret guarantee, it causes $t_{min}$ to increase, potentially
compromising the performance for shorter horizons. As before, prior
information on, say, moment bounds satisfied by the arm distributions
can be incorporated into the design of~$f.$ For example, if it is
believed that $\nu \in \G(\epsilon,B),$ a natural choice of $f$ would be
$f(t) = c + h(t),$ where $h(\cdot)$ is a slow growing function, and
$c > 1$ is the smallest constant satisfying:
$\log(x) \leq x^{\epsilon}/3B$ for all $x \geq c;$ this choice would
make $t_{min}$ close to zero for instances in $\G(\epsilon',B')^k,$
for $\epsilon' \approx \epsilon,$ $B' \approx B$ (see Appendix
\ref{proof:sr-ucb-2}). The proof of Theorem \ref{thm:sr-ucb-2} is
provided in Appendix \ref{proof:sr-ucb-2}.


\section{Experimental Analysis}
\label{expt-analysis}


\begin{figure}[t!]
\centering
        \begin{subfigure}[b]{0.3\textwidth}
                \includegraphics[width=\linewidth]{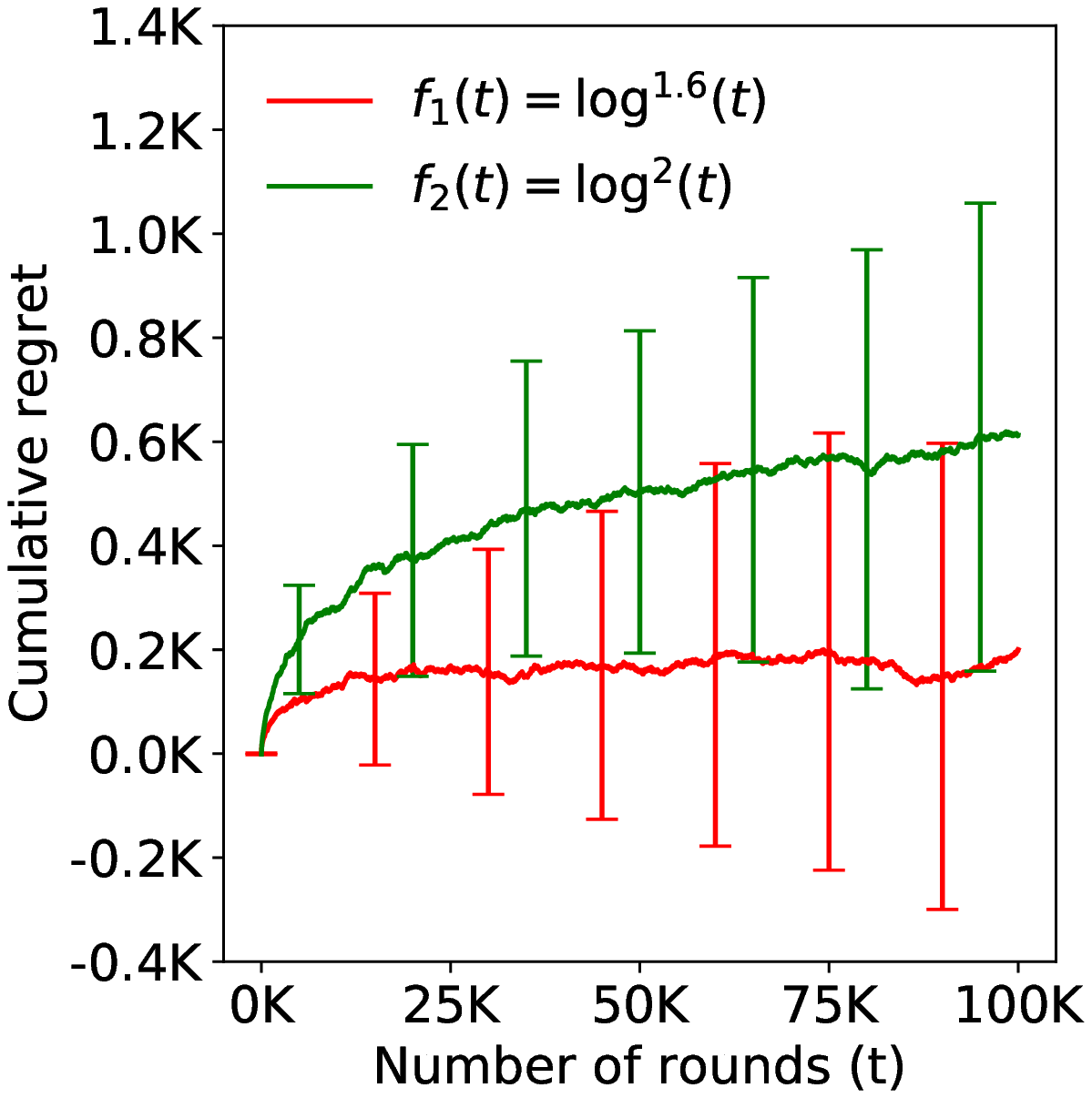}
                \caption{R-UCB: Comparison between  different scaling functions}
                \label{fig:R-UCB}
        \end{subfigure}\hspace{0.1 in}
        \begin{subfigure}[b]{0.3\textwidth}
                \includegraphics[width=\linewidth]{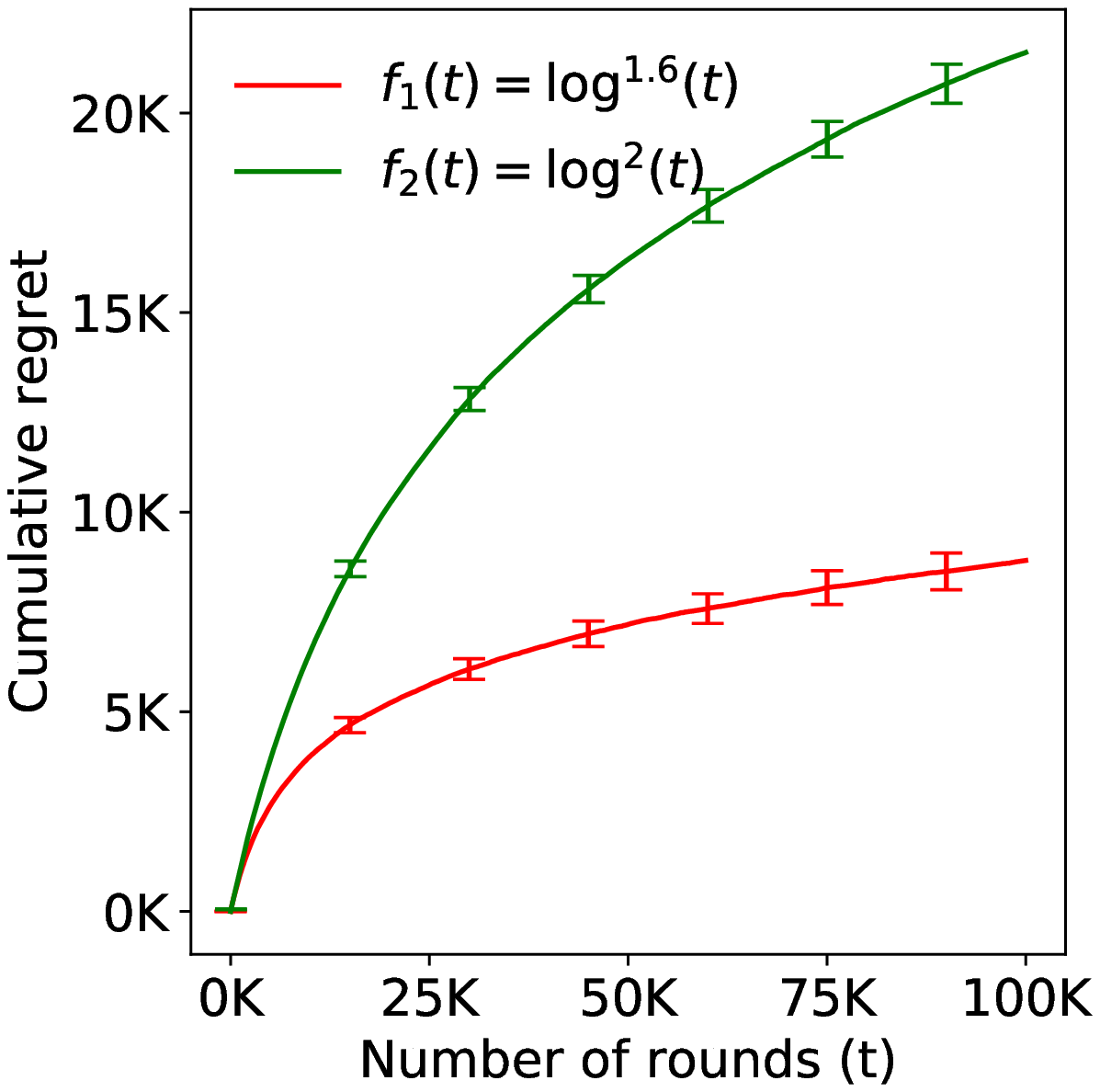}
                \caption{R-UCB-G: Comparison between different scaling functions}
                \label{fig:R-UCB-G}
        \end{subfigure}\hspace{0.1 in}
        \begin{subfigure}[b]{0.3\textwidth}
                \includegraphics[width=\linewidth]{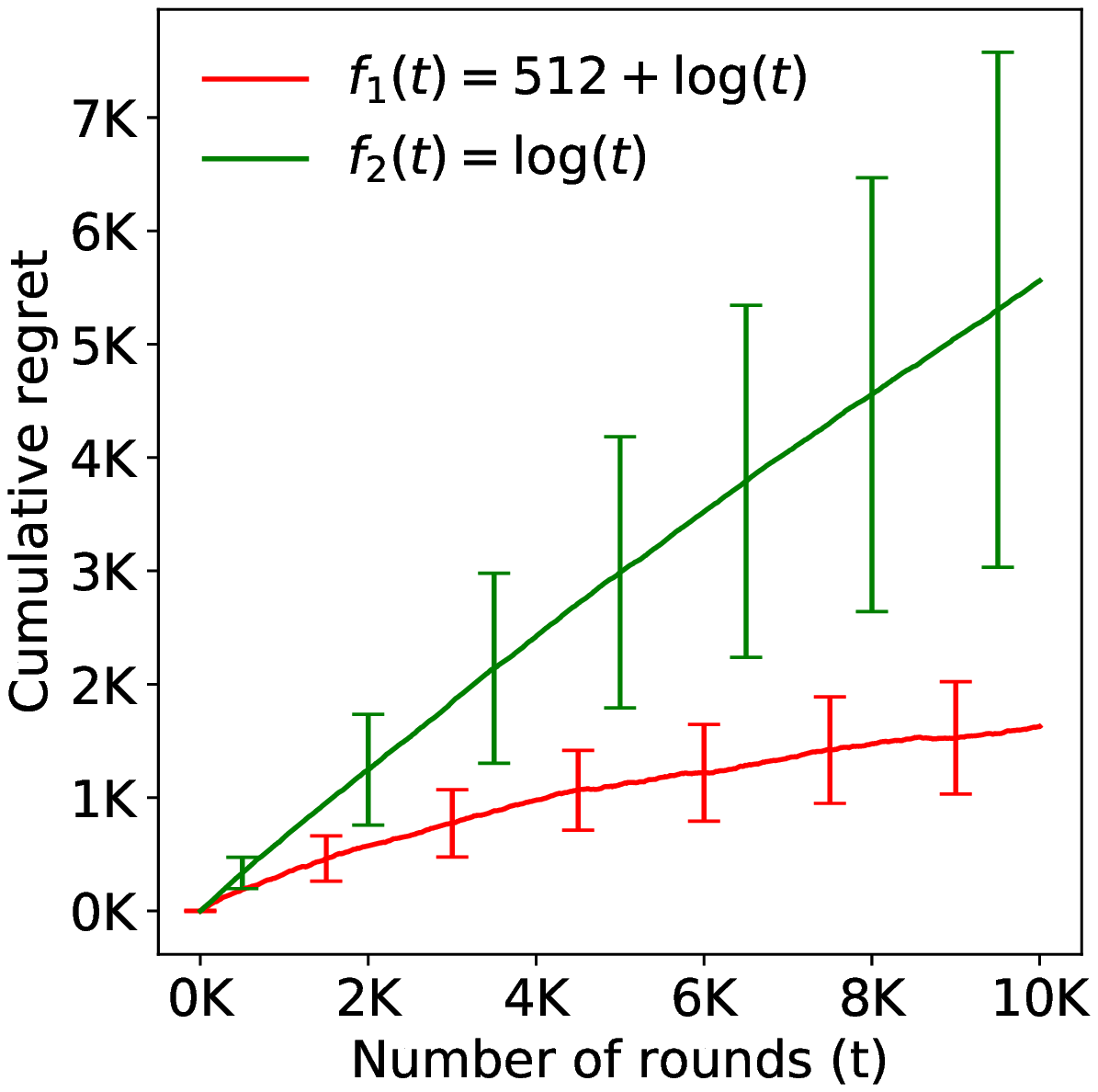}
                \caption{Using prior information to improve short-horizon regret}
                \label{fig:robust_comparison}
        \end{subfigure}%
        \label{fig:figures_main}
\end{figure}

In this section, we present numerical results to illustrate the
performance of the algorithms presented in Section~\ref{upper-bounds}.

In the first experiment, we demonstrate the effect of choice of
scaling function~$f$ on the cumulative regret. As per
Theorems~\ref{thm:sr-ucb-1} and~\ref{thm:sr-ucb-2}, the regret grows
faster asymptotically if we choose a faster growing~$f$. We
demonstrate this behavior for R-UCB and R-UCB-G in
Figures~\ref{fig:R-UCB} and~\ref{fig:R-UCB-G} respectively. The chosen
instance is as follows: two arms both distributed as Gaussian
$\mathcal{N}(\mu, \sigma)$ with parameters $(1.7, 1)$ and $(3.7,
3)$. This choice of parameters is arbitrary and a similar trend was
observed in trials with other Gaussian instances. The two chosen
scaling functions are $f_1(t) = \log^{1.6}(t)$, and
$f_2(t) = \log^{2}(t)$. The simulation is repeated 200 times for each
configuration and the empirical mean is plotted along with the
standard deviation in Figures~\ref{fig:R-UCB} and~\ref{fig:R-UCB-G}
for R-UCB and R-UCB-G, respectively. We note that the observed
cumulative regret corresponding to the faster growing $f_2(t)$ exceeds
that corresponding to $f_1(t)$ in both cases. Interestingly, this
dominance holds even for smaller horizon values, even though our
regret bounds suggests that tuning the scaling function for better
asymptotic performance might compromise short-horizon regret. This is
because our regret bounds (and UCB upper bounds in the literature most
generally) are fairly loose. Indeed, we also observe that the
cumulative regret in all the cases is well below the bounds presented
in Theorems~\ref{thm:sr-ucb-1} and~\ref{thm:sr-ucb-2}.  Also, the
regret of R-UCB is less than R-UCB-G for the same choice of~$f(t),$
which is reasonble considering we have used a light-tailed instance.

In the second experiment, we demonstrate how choosing $f(t)$ based on
(noisy) prior information can decrease regret over short horizons.
The chosen instance for this experiment is as follows: two arms both
distributed as Gaussian $\mathcal{N}(\mu, \sigma)$ with parameters
$(0, 1)$ and $(1, 10)$.
Now, suppose we have the (noisy) prior information that the arms are
$\sigma$-subGaussian with $\sigma \approx 8$. As stated in
Section~\ref{upper-bounds}, we incorporate this prior information into
the design of $f(t)$ by choosing $f_1(t) = 512 + \log(t)$. We compare
the cumulative regret for this choice with that corresponding to a
completely oblivious choice of $f(t)$, i.e., $f_2(t) = \log(t)$. The
experiment is repeated 200 times and obtained mean and standard
deviation of regret is shown in Figure \ref{fig:robust_comparison}. We
can see that $f_1(t)$, i.e., the scaling function chosen based on the
prior information, incurs lower regret.
This trend in cumulative regret can be reasoned as follows. The
algorithm using scaling function $f_2(t)$ uses smaller confidence
widths, which results in greater susceptibility to the noise in the
arm rewards.
In conclusion, if noisy prior information about the possible arm
distributions is available, this can be incorporated into the choice
of the scaling function to improve short-horizon performance, while
retaining statistical robustness.

\ignore{There are four plots. First, two figures analyze the variation
  of the threshold horizon with respect to the parameters of the
  underlying distribution. In the first plot, the variance is varied
  and the threshold horizon is plotted. The threshold is decided based
  on max between two terms, and the dotted vertical line represents a
  switch in the dominating term. It is important to note that before
  that switching point, the growth in the threshold is of the order of
  "f-inverse". However, after that switching point, the threshold
  grows doubly exponentially ($e^(f-inverse)$ to be precise). In
  conclusion, if we have a rough estimate of the parameters, we can
  choose an "f" such that the threshold horizon is less than what we
  can afford to have.  Plot 2 does a similar analysis by varying the
  sub-optimality gap. One interesting point in this is that there is
  an additional inequality concerning $\Delta_i$ due to which the
  threshold horizon for low $\Delta_i$ shoots up.

  The last two plots are regret plot averaged over 200 runs, first
  with light-tail algorithm and then heavy-tail algorithm. Heavy-tail
  has higher regret in this case (not true in general, heavy-tail will
  have lower regret when $\Delta_i < 0$). The standard deviation is
  high in both cases.}

\section{Concluding remarks}

In this paper, we demonstrated the fundamental trade-off between
logarithmic regret and statistical robustness in stochastic MABs. We
also proposed robust algorithms that incur slightly super-logarithmic
regret. 
It would be interesting to explore similar trade-offs between
statistical robustness and performance in other bandit settings,
including thresholding bandits \citep{Locatelli2016}, linear bandits
\citep{Rusmevichientong2010} and combinatorial bandits
\citep{Chen2013}. 

More broadly, we hope that this paper spawns further work on
statistically robust online learning algorithms. We have focussed on
one of the simplest learning paradigms (regret minimization in MABs),
where a logarithmic regret emerged as a robustly unattainable
performance barrier. Other fundamental performance barriers of
statistically robust learning await discovery, in more challenging
settings such as Markovian bandits and Markov Decision Processes.

\section*{Broader Impact}
This work does not present any foreseeable ethical or societal consequences.

\begin{ack}
Acknowledgment 
\end{ack}
\bibliography{ref_neurips}
\newpage
\begin{appendices}

\section{Appendix for Section 3 - Impossibility of logarithmic regret for statistically robust algorithms}
\label{lower-bounds-appendix}

This section is devoted to the proof of
Theorem~\ref{lower-bound-main-statement}. The proof is based on the
following characterization of instance-dependent lower bounds from
\citet{lattimore2018bandit} (see Theorem 16.2): 
\begin{theorem}
    For any algorithm
$\pi$ that is consistent over $\M^k,$ and instance $\nu \in \M^k,$
\begin{align*}
  \lim_{n \to \infty} \inf \frac{R_n(\pi,\nu)}{\log(n)} \geq \sum_{i: \Delta_i>0} \frac{\Delta_i}{d_i(\nu_i, \mu^*, \M)},
\end{align*}
where
$d_i(\nu_i, \mu^*, \M) := \inf_{\nu'_i \in \M}\{D(\nu_i,\nu'_i):\
\mu(\nu'_i) > \mu^*\}.$
\end{theorem}

The proof of Theorem~\ref{lower-bound-main-statement} therefore
follows from the following lemma, which shows that
$d_i(\nu_i, \mu^*, \M) = 0$ for all suboptimal arms of any instance
$\nu$ when $\M$ is $\B,$ $\SG,$ $\SE,$ or $\G.$
\begin{lemma}
\label{d_zero_lemma}
Fix $\M \in \{\B,\SG, \SE, \G\}.$ For any distribution $F \in \M,$ and
for any $a>0$ and $b>\mu(F)$, there exists distribution
$F' \in \M$ such that
\begin{align*}
  D(F,F') \leq a \quad \text{ and } \quad \mu(F') \geq b.
\end{align*}
\end{lemma}

\begin{proof}
  
  We consider the following two cases.
  
\textbf{Case 1:} $\M \in \{ \SG, \SE, \G \}$

If the distribution $F$ is unbounded from above (i.e.,
$\bar{F}(y) > 0$ for all $y \in \R$), then the claim follows from
Lemma 1 in \citet{sandeep}. The idea there is to construct a new
distribution $F'$ such that for a chosen $y$, the CDF on the left
side is decreased by a factor of $e^{-a}$ with respect to $F,$
and rest of the mass is pushed on the right side of $y$. Crucially,
under this perturbation, $F'$ remains in $\M$, since on both sides of $y$ only a constant is being multiplied, thus keeping the functional form of the distribution same. The KL-divergence
$D(F,F')$ is always less than $a$ independent of the choice
of $y$. However, the mean of $F'$ can be made arbitrary large by
choosing a suitably large value of $y$.

On the other hand, if $F$ is bounded from above, then the
argument below (for the case $\M = \B$) can be applied to construct
$F'$ that is also bounded from above, but satisfies the
conditions required. (Specifically, the boundedness of the lower
end-point of the support is not required for this argument.)

\textbf{Case 2:} $\M = \B$

We construct a new bounded distribution $F'$ such that the CDF of
$F'$ is $e^{-a}$ times the CDF of $F$ over its support. The
rest of the probability mass is uniformly distributed starting from
the right end-point of the support to an arbitrary point $v'$.

Suppose that the support of $F$ is contained within $[u,v].$
Define the CDF of distribution $F'$ as follows, for $\gamma \in (0,1)$
and $v' > v.$
\begin{align*}
  F'(x) &= (1-\gamma) F(x) \qquad \forall \ x \leq v\\
  F'(x) &= 1 + \gamma \frac{x-v'}{v'-v} \qquad \forall \ x \in (v,v']
\end{align*}

Now, 
\begin{align*}
  D(F, F') = \int_{u}^{v} \log \left(\frac{dF(x)}{dF'(x)}\right) dF(x) = -\log(1-\gamma).
\end{align*}
Choosing $\gamma = 1-e^{-a}$ yields $D(F, F') =
a$. Turning now to the mean of $F',$
\begin{align*}
	\mu(F') &= \int_{u}^{v'} x dF'(x) = (1-\gamma) \mu(F) + \int_{v}^{v'} x \frac{\gamma}{v'-v} dx\\
	&= (1-\gamma) \mu(F) + \frac{\gamma}{2} (v'+v)
\end{align*}
Clearly,
$\mu(F')$ can be made arbitrarily large by choosing a suitably
large $v'.$

\end{proof}

\section{Proof of Theorem \ref{thm:sr-ucb-1} - Regret Upper Bound for R-UCB}
\label{proof-sr-ucb-1}

We formally prove theorem \ref{thm:sr-ucb-1} in this section. The
prove is structurally similar to the bandit regret proof presented in
\citet*{bubeck}. We will show regret bound for the two cases
$\nu \in \SG^k$, and, $\nu \in \SE^k$.
\begin{proof} We first prove for $\nu \in \SE^k$ and then the other case follows.

\textbf{Case 1} ~~ $\nu \in \SE^k$

We define the following three events for any sub-optimal arm $i$.

\begin{align*}
	E_1~&: \qquad \U(i^*, T_{i^*}(t-1), t) \leq \mu^*\\
	E_2&: \qquad \hat{\mu}(i, T_{i}(t-1)) > \mu_i + \W(T_{i}(t-1), t)\\
	E_3&: \qquad \Delta_i < 2\W(T_{i}(t-1), t)
\end{align*}

where $T_i(t)$ denotes the number of times $i^{th}$ arm is pulled till time instant $t$. The three events can be interpreted as follows. Event $E_1$ occurs when the upper confidence bound corresponding to the optimal arm is less than its actual mean. Event $E_2$ corresponds to the case when the mean estimator of a sub-optimal arm is much more than its actual mean. As we shall see, both $E_1$ and $E_2$ are low-probability event and its probability can be upper bounded. Finally, event $E_3$ corresponds to the case when the confidence window of arm $i$ is large. We now prove that one of these event must be true when a sub-optimal arm is chosen at time instant $t$. Denote $I_t$ as the arm chosen at time $t$.

\emph{Claim} ~~ If $I_t = i$, then one of $E_1, E_2$ or $E_3$ is true.

To justify this claim, we assume all the three events to be false and then show a contradiction.

We have,
\begin{align*}
	\U(i^*, T_{i^*}(t-1), t) &> \mu^*\\
	&= \mu_i + \Delta_i\\
	&\geq \mu_i + 2\W(T_{i}(t-1), t)\\
	&\geq \hat{\mu}(i, T_{i}(t-1)) + \W(T_{i}(t-1), t)\\
	&= \U(i, T_{i}(t-1), t)
\end{align*}
which is a contradiction since $I_t \neq i^*$.

We now show a distribution oblivious concentration inequality for each $\nu \in \SE^k$. This inequality will be useful in upper bounding probability of events $E_1$ and $E_2$.

By our choice of algorithm

$$\hat{\mu}(i, u) = \frac{1}{u} \sum_{j=1}^{u} X_j ~\text{;}~~~~ \W(u, t) = \sqrt{\frac{f(t)\log(t)}{u}} $$

We assume the underlying distribution to be $\nu \in \SE(v, \alpha)^k$. For any confidence width $\W$, we have the following concentration inequality (see equation 2.18 in \citet*{wainwright2019high})

$$ \mathbb{P}\left( \frac{1}{u} \sum_{j=1}^{u} X_j - \mu \geq \W \right) \leq \text{exp} \left( - \min \left\{ \frac{u\W^2}{2v^2}, \frac{u\W}{2\alpha} \right \} \right) $$

We are interested only in small values of the confidence window $\W$, and hence the first term in the minimum expression is of interest to us. For the first term to be less than the second term, we have the following inequality
$$ \W \leq \frac{v^2}{\alpha} $$

Putting the value of confidence window $\W(u, t)$ in this inequality, we get,
$$ u \geq f(t)\log(t) \left( \frac{\alpha}{v^2} \right)^2 $$

Denote the minimum $u$ satisfying this inequality as $u_0$. Hence for all $u>u_0$ we have,

$$ \mathbb{P}\left( \hat{\mu}(i^*, u) + \W(u, t) > \mu^* \right) \leq  \text{exp} \left( \frac{-f(t)\log(t)}{2v^2} \right) $$

Since $f(t)$ is a sub-linearly growing function, for all time $t>t_0$, we are guaranteed to have $f(t)>8v^2$, where $t_0 = f^{-1}(8v^2)$. Substituting this inequality in the above expression yields,

$$ \mathbb{P}\left( \hat{\mu}(i^*, u) + \W(u, t) > \mu^* \right) \leq  \text{exp} \left( -4\log(t) \right) = t^{-4} $$

This expression establishes a distribution oblivious inequality for subexponential random variables. This inequality is valid for all time instances $t>t_0$ and $u > u_0$, where $t_0$ is a distribution dependent constant parameter while $u_0$ depends on the distribution as well as the choice of $f$. In addition, $u_0$ is an increasing function with number of rounds $t$.

This inequality is useful in establishing an upper bound on the probability of events $E_1$ and $E_2$. We have,

\begin{align*}
\mathbb{P}(E_1) &\leq \mathbb{P}(\exists u \in [t]: \U(i^*, u, t) \leq \mu^*) \leq t.t^{-4} = t^{-3} ~~ \text{by union bound over } u
\end{align*}

Similarly, $\mathbb{P}(E_2) \leq t^{-3}$.


Let $u'_i$ denote the maximum value of $T_i(t-1)$ for which event $E_3$ is true. Consequently, for all $t>u'_i$ and $u>u_0$, if $I_t=i$, then at least one of the event $E_1, $ $E_2$ is true. Finally, we choose $u_i = \max (u'_i, u_0, t_0)$ since we wish to apply the above concentration inequality for all time instances $t>u_i$.

Now, for any sub-optimal arm $i$,

\begin{align*}
\mathbb{E}[T_i(t)] &= \mathbb{E} \left[ \sum_{s=1}^{t} \mathbbm{1}\{ I_t = i \} \right]\\
& \leq u_i + \mathbb{E} \left[ \sum_{s=u_i+1}^t \mathbbm{1}\{I_t = i\} \right]\\
&= u_i + \mathbb{E} \left[ \sum_{s=u_i+1}^t \mathbbm{1}\{ I_t = i, E_1 \text{ true or } E_2 \text{ true} \} \right]\\
& \leq u_i + \sum_{s=u_i+1}^t \mathbb{P}(E_1 \cup E_2)\\
& \leq u_i + \sum_{s=u_i+1}^t \frac{2}{s^3} \leq u_i + 4
\end{align*}

Evaluating the value of $u_i$, we get

$$ u_i = \max \left \{ \frac{4f(t)\log(t)}{\Delta_i^2}, f(t)\log(t) \left( \frac{\alpha}{v^2} \right)^2, t_0 \right \}$$

%

However, we observe that $t_0$ is a constant and thus the first two terms ($u'_i, u_0$) will be more than $t_0$ after a time instance, say $t_1$. Hence,
$$ \mathbb{E}[T_i(t)] \leq \max\left\{\frac{4f(t)\log(t)}{\Delta_i^2}, f(t)\log(t)\left(\frac{\alpha}{v^2}\right)^2\right \} + 4 \quad \forall t > t^{\SE}_{min}(\nu) $$

where the instance dependent threshold $t^{\SE}_{min}(\nu) = \max(t_0, t_1)$.

Thus, we get the regret upper bound as 
$$  \boxed{R_t(\nu) \leq \sum_{i:\Delta_i>0} \left( f(t)\log(t) ~ \max \left\{ \frac{4}{\Delta_i}, \Delta_i \left( \frac{\alpha}{v^2} \right)^2 \right\} + 4\Delta_i \right) \quad \forall t>t^{\SE}_{min}(\nu)}$$  

\textbf{Case 2} ~~ $\nu \in \SG^k$

We observe that, $\SG$ is a special case of $\SE$ with $\alpha \rightarrow 0$. And hence, the regret expression can be obtained as

$$  \boxed{R_t(\nu) \leq \sum_{i:\Delta_i>0} \left( \frac{4f(t)\log(t)}{\Delta_i} + 4\Delta_i \right) \quad \forall t>t^{\SG}_{min}(\nu)}$$  

where the instance dependent threshold $t^{\SG}_{min} = \max(t_0, t_1)$ with $t_0$ and $t_1$ same as the previous case.

\end{proof}

\subsection{Regret bounds when $t<t_{min}$}
\label{ssec:suboptimal-bound}
We discuss a weaker regret bound for time instances less than the threshold time $t_{min}$. In the proof of theorem \ref{thm:sr-ucb-1} above, we use a slow increasing scaling function to make the inequality oblivious to its parameters. However, we are also interested in obtaining a regret bound for $t<t_{min}$. We have,
$$ \mathbb{P}\left( \hat{\mu}(i^*, u) + \W(u, t) > \mu^* \right) \leq  \text{exp} \left( -\hat{c} f(t)\log(t) \right) $$
where 
\[
    \hat{c} = 
\begin{cases}
    \frac{2}{(b-a)^2},& \text{if } \nu \in \B^k\\
    \frac{1}{2\sigma^2},& \text{if } \nu \in \SG^k\\
    \frac{1}{2v^2}, &\text{if } \nu \in \SE^k
\end{cases}
\]

Substituting this weaker concentration bound in the above proof of regret bound we get,
$$ \mathbb{E}[T_i(t)] \leq u_i + \sum_{s=u_i+1}^{t} t^{1-\hat{c}f(t)\log(t)} $$ as the expected number of times a sub-optimal arm is pulled. The above expression for $\mathbb{E}[T_i(t)]$ still yields a \emph{sub-linear} upper bound, though weaker than before.

\section{Proof of Theorem \ref{thm:sr-ucb-2} - Regret Upper Bound for R-UCB-G}
\label{proof:sr-ucb-2}

We prove theorem \ref{thm:sr-ucb-2} in this section. This proof is similar to proof of theorem \ref{thm:sr-ucb-1} given in appendix \ref{proof-sr-ucb-1}.

\begin{proof}
We define the following three events for any sub-optimal arm $i$.

\begin{align*}
	E_1~&: \qquad \U(i^*, T_{i^*}(t-1), t) \leq \mu^*\\
	E_2&: \qquad \hat{\mu}(i, T_{i}(t-1), t) > \mu_i + \W(T_{i}(t-1), t)\\
	E_3&: \qquad \Delta_i < 2\W(T_{i}(t-1), t)
\end{align*}

where $T_i(t)$ denotes the number of times $i^{th}$ arm is pulled till time instant $t$. The three events can be interpreted as follows. Event $E_1$ occurs when the upper confidence bound corresponding to the optimal arm is less than its actual mean. Event $E_2$ corresponds to the case when the mean estimator of a sub-optimal arm is much more than its actual mean. As we shall see, both $E_1$ and $E_2$ are low-probability event and its probability can be upper bounded. Finally, event $E_3$ corresponds to the case when the confidence window of arm $i$ is large. We now prove that one of these event must be true when a sub-optimal arm is chosen at time instant $t$. Denote $I_t$ as the arm chosen at time $t$.

\emph{Claim} ~~ If $I_t = i$, then one of $E_1, E_2$ or $E_3$ is true.

To justify this claim, we assume all the three events to be false and then show a contradiction.

We have,
\begin{align*}
	\U(i^*, T_{i^*}(t-1), t) &> \mu^*\\
	&= \mu_i + \Delta_i\\
	&\geq \mu_i + 2\W(T_{i}(t-1), t)\\
	&\geq \hat{\mu}(i, T_{i}(t-1), t) + \W(T_{i}(t-1), t)\\
	&= \U(i, T_{i}(t-1), t)
\end{align*}
which is a contradiction since $I_t \neq i^*$.

Now, by our choice of algorithm

$$ \hat{\mu}(i, u, t) = \frac{1}{u} \sum_{j=1}^{u} X_j \mathbbm{1}_{ \left \{ |X_j| \leq f(t) \right \} }$$

We attempt to establish a distribution oblivious concentration inequality with mean estimator chosen as $\hat{\mu}(i, u, t)$. We draw inspiration from already established non-oblivious concentration inequality based on this mean estimator (see Lemma 1 in \citet*{bubeck}, Lemma 1 in \citet*{Yu2018PureEO} which uses results from \citet*{selden2012}).

We assume the underlying instance to be in $G(\epsilon, B)^k$. For a truncation parameter $f(t)$, we have, with a probability at least $1-t^{-4}$

\begin{align*}
\mu - \hat{\mu}(i, u, t) &\leq \frac{B}{f(t)^{\epsilon}} + \frac{1}{u} \left( 2f(t)\log(2t^4) + u \frac{B}{2f(t)^{\epsilon}} \right)\\
&\leq \frac{3B}{2f(t)^{\epsilon}} + \frac{16f(t)\log(t)}{u}
\end{align*}

Now, the only non-obliviousness is due to the first term. We observe that, for all $t>t_0$, $3B\log(f(t)) < 2f(t)^{\epsilon}$. There always exists $t_0$ such that this is true, since, left hand side is a sub-linear term, while right hand side is not.

For all $t>t_0$, with a probability at least $1-t^{-4}$

$$\mu - \hat{\mu}(i, u, t) \leq \frac{1}{\log(f(t))} + \frac{16f(t)\log(t)}{u}$$

$$\Rightarrow \mathbb{P}\left( \mu - \hat{\mu}(i, u, t) \geq \W(u, t) \right) \leq t^{-4}$$

This expression establishes a distribution oblivious inequality for a general (even heavy-tailed) random variables. This inequality is valid for all time instances $t>t_0$, where $t_0$ is a distribution dependent constant parameter.

This inequality is useful in establishing an upper bound on the probability of events $E_1$ and $E_2$, similar to case 1 in the proof given in Appendix \ref{proof-sr-ucb-1}. We have,

\begin{align*}
\mathbb{P}(E_1) &\leq \mathbb{P}(\exists u \in [t]: \U(i^*, u, t) \leq \mu^*) \leq t.t^{-4} = t^{-3} ~~ \text{by union bound over } u
\end{align*}

Similarly, $\mathbb{P}(E_2) < t^{-3}$.

Now, we proceed to obtain regret upper bound similar to case 1 in the proof given in Appendix \ref{proof-sr-ucb-1}. We define $u'_i$ as the maximum value of $T_i(t-1)$ for which event $E_3$ is true. Also, we wish to apply concentration bound for all time instants $t>u_i$. Consequently, we choose $u_i = \max(u'_i, t_0)$.

Similar to the previous case, we get,

$$\mathbb{E}[T_i(t)] \leq u_i + 4$$

The value of $u_i$ can be evaluated from the inequality given in event $E_3$ and the choice of $\W(u, t)$. We get,

$$u_i = \max\left\{ \frac{32f(t)\log(t)}{\Delta_i - \frac{2}{\log(f(t))}}, t_0 \right\}$$.

However, the above calculated value of $u'_i$ is valid only when

$$\Delta_i - \frac{2}{\log(f(t))} > 0$$

Let $t_1$ denote the minimum value of $t$ satisfying the equation above. Moreover, we observe that $t_0$ is a constant and thus the first term in the expression of $u_i$ will be more than $t_0$ after a time instance, say $t_2$. Hence,

$$\mathbb{E}[T_i(t)] \leq \frac{32f(t)\log(t)}{\Delta_i - \frac{2}{\log(f(t))}} \quad \forall t > t_{min}(\nu)$$

where the instance dependent threshold $t_{min} = \max(t_0, t_1, t_2)$.

Thus, we get the regret upper bound as 

$$ \boxed{ R_t(\nu) \leq \sum_{i:\Delta_i>0} \left( \frac{32f(t)\log(t)}{1 - \frac{2}{\Delta_i \log(f(t))}} + 4\Delta_i  \right) \quad \forall t>t_{min}(\nu)} $$

\end{proof}

%
%

\section{Robust Upper Confidence Bound algorithm for arbitrary instances using Median of Means (MoM) estimator}
\label{mom-section}

Similar to R-UCB-G algorithm, we present yet another statistically robust algorithm over $\G^k$. Instead of truncation-based estimator, we use median of means estimator (see \citet{bubeck}) . This estimator works well under excessive variability in the sample values. The mean estimator in MoM works as follows. The samples are first divided into $q$ bins each having equal number of samples. Empirical mean is calculated for each of the bins and the median of $q$ mean values is the mean estimator of the samples. In truncation-based estimator, high sample values will require high truncation value in order to contribute to the mean estimator. For such excessive variable samples, the proposed algorithm, R-UCB-G-MoM will have slightly better finite horizon performance.

\begin{algorithm}[t]
  \caption{R-UCB-G-MoM}
  \label{mom-algo}
  \textbf{Input} $k$ arms, slow growing scaling function $f$, slowly decaying function $g$
  \begin{algorithmic}
     \State \textbf{Initialize} $\mathcal{R}_i = \{~\}$, $u_i=0$ for all arm $i$
    \For  {$t=1$ to $k$}
    \State pull arm with index $i = t-1$ and observe reward $R_t$
    \State Append $R_t$ to $\mathcal{R}_i$ and update $u_i \leftarrow u_i + 1$
    \EndFor 
    \For { $t=k+1, k+2, \dots$ }
    \State Calculate mean estimator $\hat{\mu}(i, u, t)$ using algorithm \ref{mom-function} with input $\mathcal{R}_i$, $u_i$ and $t$
    \State Calculate the upper confidence bound as $$\U(i, u_i, t) = \hat{\mu}(i, u, t) + \underbrace{f(t)\left(\frac{32\log(t)}{u}\right)^{g(t)}}_{\W(u_i, t)} $$
	\State Pull arm $i$ maximizing $\U(i, u_i, t)$ and observe reward $R_t$
	\State Append $R_t$ to $\mathcal{R}_i$ and update $u_i \leftarrow u_i+1$
    \EndFor  
  \end{algorithmic}
\end{algorithm}

\begin{algorithm}[t]
  \caption{Function to Calculate Median of Means (MoM)}
  \label{mom-function}
   \textbf{Input} $\mathcal{R}$, $u$, $t$.
  \begin{algorithmic}
    \If {$u>32\log(t)$}:
      \State Take $q = \ceil{32\log(t)}$ and $N = \ceil{\frac{u}{q}}$
      \State Compute $\hat{\mu}_l = \frac{1}{N} \sum_{m=1}^{N} R_{\{(l-1)N + m\}}$ for $l=1,2,\dots,q$
      \State \Return median($\hat{\mu}_1, \hat{\mu}_2, \dots, \hat{\mu}_q$)
    \Else:
      \State \Return median($\mathcal{R}$)
    \EndIf
  \end{algorithmic}
\end{algorithm}

In addition to scaling function $f$ put down in Definition \ref{ref:f}, we need another class of functions in this algorithm, which is stated as follows.

\begin{definition}
  A function~$g:\mathbb{N} \ra (0,\infty)$ is said to be \emph{slow
    decaying} if
  $$g(t + 1) \leq g(t)\ \forall\ t \in \N, \quad
  \lim_{t \ra \infty} g(t) = 0, \quad \lim_{t \ra \infty}
  \frac{g(t)}{t^a} = 0 \ \forall \ a > 0.$$
\end{definition}

R-UCB-G-MoM provides the following regret guarantee over instances in $\G^{k}$.

\begin{theorem}
  \label{thm:sr-ucb-3}
  Consider the algorithm R-UCB-G-MoM with a specified slow growing scaling
  function~$f$ and slow decaying function~$g$. For an instance $\nu \in \mathcal{G}(\epsilon,B)^k,$
  there exists a threshold $t_{min}(\epsilon,B)$ such that for
  $t > t_{min}(\epsilon,B),$ the regret under R-UCB-G-MoM satisfies
  $$R_t(\nu) \leq \sum_{i:\Delta_i>0} \left( \Delta_i \left(\frac{2f(t)}{\Delta_i} \right)^{\frac{1}{g(t)}}32\log(t)  + 4\Delta_i \right).$$
\end{theorem}

The proof is similar to proof of theorem \ref{thm:sr-ucb-1} presented in appendix \ref{proof-sr-ucb-1}.

\begin{proof} We define the following three events for any sub-optimal arm $i$.

\begin{align*}
	E_1~&: \qquad \U(i^*, T_{i^*}(t-1), t) \leq \mu^*\\
	E_2&: \qquad \hat{\mu}(i, T_{i}(t-1), t) > \mu_i + \W(T_{i}(t-1), t)\\
	E_3&: \qquad \Delta_i < 2\W(T_{i}(t-1), t)
\end{align*}

where $T_i(t)$ denotes the number of times $i^{th}$ arm is pulled till time instant $t$. The three events can be interpreted as follows. Event $E_1$ occurs when the upper confidence bound corresponding to the optimal arm is less than its actual mean. Event $E_2$ corresponds to the case when the mean estimator of a sub-optimal arm is much more than its actual mean. As we shall see, both $E_1$ and $E_2$ are low-probability event and its probability can be upper bounded. Finally, event $E_3$ corresponds to the case when the confidence window of arm $i$ is large. We now prove that one of these event must be true when a sub-optimal arm is chosen at time instant $t$. Denote $I_t$ as the arm chosen at time $t$.

\emph{Claim} ~~ If $I_t = i$, then one of $E_1, E_2$ or $E_3$ is true.

To justify this claim, we assume all the three events to be false and then show a contradiction.

We have,
\begin{align*}
	\U(i^*, T_{i^*}(t-1), t) &> \mu^*\\
	&= \mu_i + \Delta_i\\
	&\geq \mu_i + 2\W(T_{i}(t-1), t)\\
	&\geq \hat{\mu}(i, T_{i}(t-1), t) + \W(T_{i}(t-1), t)\\
	&= \U(i, T_{i}(t-1), t)
\end{align*}
which is a contradiction since $I_t \neq i^*$.

Now, by our choice of algorithm $\hat{\mu}(i, u, t)$ is the \emph{median of means} estimator. In this mean estimator, we first divide the samples into $q$ bins, and compute the average of all the bins. Each bin will have $N = \ceil{\frac{u}{q}}$ samples. We return the median of these $q$ bins as the mean estimator. We attempt to establish a distribution oblivious concentration inequality for this mean estimator. Formally, this estimator is defined as
$$ \hat{\mu}(i, u, t) = \text{median}(\hat{\mu}_1, \hat{\mu}_2, \dots, \hat{\mu}_q) \quad \text{where } q = \ceil{32\log(t)} ~~ \text{ and } ~~ \hat{\mu}_l = \frac{1}{N} \sum_{m=1}^N X_{\{(l-1)N+m\}}  $$

The choice of $q = \ceil{32\log(t)}$ is useful in establishing the required concentration inequality. This requirement comes from the fact that we need at least $N=1$ samples per bin. Further, we assume that for all arms, $u > 32\log(t)$. Hence, the inequality that we now propose is valid only for $u > 32 \log(t)$.

We define a bernoulli random variable $Y_l = \mathbbm{1}\{ \hat{\mu}_l > \mu + \W \}$. According to equation 12 in \citet*{bubeck}, $Y_l$ has the parameter
$$ p \leq \frac{3B}{N^{\epsilon}\W^{1+\epsilon}} $$ 

Choosing $\W(u, t) = f(t) \left( \frac{1}{N} \right)^{g(t)}$, where $f(t)$ is a slow growing function, and $g(t)$ is a slow decaying function, yields,

$$p \leq \frac{3B}{N^{\epsilon}f(t)^{1+\epsilon}\left(\frac{1}{N}\right)^{g(t)(1+\epsilon)}}$$

Since $f(t)$ is slow growing and $g(t)$ is slow decaying, we are guaranteed to have a $t_0$ such that, for all $t>t_0$, we have $g(t)<\frac{\epsilon}{1+\epsilon}$ and $f(t)^{1+\epsilon} > 12B$. For such $t>t_0$, we get,
\begin{align*}
    p &\leq \left( \frac{1}{4} \right) \left( \frac{3B}{12f(t)^{1+\epsilon}} \right) \left( \frac{1}{N^{\epsilon - g(t)(1+\epsilon)}} \right) \leq \frac{1}{4} 
\end{align*}

Finally, using Hoeffding inequality for binomial random variable,

\begin{align*}
\mathbb{P}\left( \hat{\mu}(i, u, t) - \mu > \W(u, t) \right) = \mathbb{P}\left( \sum_{j=1}^{q} X_j \right) &\leq \text{exp}\left( -2q(\frac{1}{2} - p)^2 \right)\\
 &\leq \text{exp}\left(\frac{-q}{8}\right) = \text{exp}\left( \frac{-32\log(t)}{8}\right) = t^{-4}
\end{align*}

Note that this inequality is valid for all time instances $t>t_0$ and $u>u_0$ where $t_0$ is a distribution dependent constant parameter and $u_0 = \ceil{32\log(t)}$, an increasing function.

This inequality is useful in establishing an upper bound on the probability of events $E_1$ and $E_2$, similar to case 1. We have,

\begin{align*}
\mathbb{P}(E_1) &\leq \mathbb{P}(\exists u \in [t]: \U(i^*, u, t) \leq \mu^*) \leq t.t^{-4} = t^{-3} ~~ \text{by union bound over } u
\end{align*}

Similarly, $\mathbb{P}(E_2) \leq t^{-3}$.

We define $u'_i$ as done in the the proof of theorem \ref{thm:sr-ucb-1}. However, for the above distribution oblivious concentration inequality to hold, we have an additional constraint of $u>u_0$. Hence, in this case we choose $u_i = \max(u'_i, u_0, t_0)$.

Similar to the previous two cases, we get,

$$\mathbb{E}[T_i(t)] \leq u_i + 4 \quad \text{ but here }~~ u_i = \max \left \{ \left(\frac{2f(t)}{\Delta_i} \right)^{\frac{1}{g(t)}}32\log(t), 32\log(t), t_0 \right \}$$

However, we observe that $t_0$ is a constant and thus the first two terms ($u'_i, u_0$) will be more than $t_0$ after a time instance, say $t'_1$. Moreover, the first function is faster growing than the second function, since $\left( \frac{2f(t)}{\Delta_i} \right)^{\frac{1}{g(t)}}$ is increasing with time instance $t$. Denote $t''_1$ as the threshold time. Define $t_1 = \max(t'_1. t''_1)$. Hence,
$$ \mathbb{E}[T_i(t)] \leq  \left(\frac{2f(t)}{\Delta_i} \right)^{\frac{1}{g(t)}}32\log(t) + 4 \quad \forall t > t_{min}(\nu) $$

where the instance dependent threshold $t_{min}(\nu) = \max(t_0, t_1)$.

Thus, we get the regret upper bound as 
$$  \boxed{R_t(\nu) \leq \sum_{i:\Delta_i>0} \left( \Delta_i \left(\frac{2f(t)}{\Delta_i} \right)^{\frac{1}{g(t)}}32\log(t)  + 4\Delta_i \right) \quad \forall t>t_{min}(\nu)}$$  

It is left to show that the above regret bound is indeed \emph{consistent}. We show that there exists appropriate choices of $f(t)$ and $g(t)$ so that the overall regret expression can be made as close to logarithmic as we want.
\end{proof}
\begin{corollary}
\label{consistency-tea}
For every slow increasing function $\Phi(t)$, there exists slow increasing function $f(t)$, slow decreasing decreasing $g(t)$ and $t_{min}$ such that $\forall ~ \Delta_i, t>t_{min}$
\begin{align*}
    \left( \frac{2f(t)}{\Delta_i} \right)^{\frac{1}{g(t)}} \leq \Phi(t) 
\end{align*}
\end{corollary} 
\begin{proof} We see that $e^{0.5(\log\Phi(t))^{1-c}}$ is an increasing function for $c \in (0, 1)$. Hence, we choose $f(t) = 0.5e^{0.5(\log\Phi(t))^{1-c}}$ and $g(t) = \frac{1}{\log^{c}(\Phi(t))}$. Also there exists $t_0$ such that for all $t>t_0$, $\frac{1}{\Delta_i} \leq e^{0.5(\log\Phi(t))^{1-c}}$ since LHS is a constant while RHS is an increasing function of $t$. Thus, we have,
\begin{align*}
    \frac{f(t)}{\Delta_i} \leq e^{(\log\Phi(t))^{1-c}} \quad \forall t > t_0
\end{align*}
Again, there exists $t_1$ such that LHS (and hence RHS) is greater than 1.

Finally for all $t>t_{min}$, where $t_{min} = \max (t_0, t_1)$, we have,
\begin{align*}
    \left( \frac{f(t)}{\Delta_i} \right)^{\frac{1}{g(t)}} \leq \left(e^{(\log\Phi(t))^{1-c}}\right)^{\log^{c}(\Phi(t))} = \Phi(t) \quad \forall t>t_{min}
\end{align*}
\end{proof}

\newpage

\end{appendices}

\end{document}